\documentclass[11pt]{article}
\usepackage[utf8]{inputenc} \usepackage[T1]{fontenc}    \usepackage{nicefrac}       \usepackage{microtype}     
\usepackage{subfig}
\usepackage{tikz}
\usepackage{dsfont}

\usepackage{pgfplots}
\usepackage{tkz-euclide}
\usetikzlibrary{arrows,shadows,positioning, calc, decorations.markings, 
	hobby, quotes,angles,decorations.pathreplacing,intersections}
\usepgfplotslibrary{polar}
\usepgflibrary{shapes.geometric}
\usepgfplotslibrary{fillbetween}

\def\colorful{1}

\usepackage{amsthm,amsfonts,amsmath,amssymb,epsfig,color,float,graphicx,verbatim, enumitem}
\usepackage{multirow}
\usepackage{algorithm}
\usepackage[noend]{algpseudocode}
\usepackage{bm}
\newif\ifhyper\IfFileExists{hyperref.sty}{\hypertrue}{\hyperfalse}
\hypertrue
\ifhyper\usepackage{hyperref}\fi

\usepackage{enumitem}

\def\nnewcolor{0}
\ifnum\nnewcolor=1
\newcommand{\nnew}[1]{{\color{red} #1}}
\fi
\ifnum\nnewcolor=0
\newcommand{\nnew}[1]{#1}
\fi

\ifnum\colorful=1

\else

\fi

\newtheorem{theorem}{Theorem}[section]

\newtheorem{lemma}[theorem]{Lemma}
\newtheorem{informal theorem}[theorem]{Theorem (informal statement)}

\newtheorem{proposition}[theorem]{Proposition}

\newtheorem{fact}[theorem]{Fact}

\theoremstyle{definition}
\newtheorem{definition}[theorem]{Definition}
\newcommand{\eqdef}{\stackrel{{\mathrm {\footnotesize def}}}{=}}

\theoremstyle{problem}
\newtheorem{problem}[theorem]{Problem}

\newcommand\snorm[2]{\left\| #2 \right\|_{#1}}
\renewcommand\vec[1]{\mathbf{#1}}

\DeclareMathOperator*{\E}{\mathbf{E}}

\def\d{\mathrm{d}}
\newcommand{\normal}{\mathcal{N}}

\newcommand{\bx}{\mathbf{x}}
\newcommand{\by}{\mathbf{y}}
\newcommand{\bv}{\mathbf{v}}
\newcommand{\bw}{\mathbf{w}}

\newcommand{\Sp}{\mathbb{S}}

\newcommand{\x}{\mathbf{x}}

\newcommand{\R}{\mathbb{R}}

\newcommand{\Z}{\mathbb{Z}}

\newcommand{\eps}{\epsilon}

\newcommand{\pr}{\mathbf{Pr}}
\newcommand{\poly}{\mathrm{poly}}
\newcommand{\polylog}{\mathrm{polylog}}

\newcommand{\sgn}{\mathrm{sign}}
\newcommand{\sign}{\mathrm{sign}}

\newcommand{\opt}{\mathrm{OPT}}
\newcommand{\D}{\mathcal{D}}

\newcommand{\Ind}{\mathds{1}}
\newcommand{\1}{\Ind}

\usetikzlibrary{angles, quotes}
\usepackage{pgfplots}

\newcommand{\dotp}[2]{\left\langle #1, #2 \right\rangle}

\newcommand{\relu}{\mathrm{ReLU}}
\newcommand{\citep}{\cite}

\title{Near-Optimal SQ Lower Bounds for Agnostically Learning Halfspaces and ReLUs
	under Gaussian Marginals}

\author{
Ilias Diakonikolas\thanks{Supported by NSF Award CCF-1652862 (CAREER), a Sloan Research Fellowship, and 
a DARPA Learning with Less Labels (LwLL) grant.}\\
University of Wisconsin-Madison\\
{\tt ilias@cs.wisc.edu}\\
\and
Daniel M. Kane\thanks{Supported by NSF Award CCF-1553288 (CAREER) and a Sloan Research Fellowship.}\\ University of California, San Diego
\\
{\tt dakane@cs.ucsd.edu}
\and
Nikos Zarifis\thanks{Supported in part by a DARPA  Learning with Less Labels (LwLL) grant.}\\
University of Wisconsin-Madison\\
{\tt zarifis@wisc.edu}\\
}

\begin{document}

\maketitle

\begin{abstract}
We study the fundamental problems of agnostically learning halfspaces and ReLUs under Gaussian marginals. 
In the former problem, given labeled examples $(\bx, y)$ from an unknown distribution on $\R^d \times \{ \pm 1\}$, 
whose marginal distribution on $\bx$ is the standard Gaussian and the labels $y$ can be arbitrary,
the goal is to output a hypothesis with 0-1 loss $\opt+\eps$, where $\opt$
is the 0-1 loss of the best-fitting halfspace. 
In the latter problem, given labeled examples $(\bx, y)$ from an unknown distribution on $\R^d \times \R$, 
whose marginal distribution on $\bx$ is the standard Gaussian and the labels $y$ can be arbitrary,
the goal is to output a hypothesis with square loss $\opt+\eps$, where $\opt$
is the square loss of the best-fitting ReLU.
We prove Statistical Query (SQ) lower bounds of $d^{\poly(1/\eps)}$ for both of these problems. 
Our SQ lower bounds provide strong evidence that current upper bounds for these tasks 
are essentially best possible.
\end{abstract}

\setcounter{page}{0}
\thispagestyle{empty}
\newpage

\section{Introduction} \label{sec:intro}

\subsection{Background and Problem Motivation} \label{ssec:background}
We study the fundamental problems of agnostically learning halfspaces
and ReLU regression in the distribution-specific agnostic PAC model.
In both of these problems, we are given i.i.d. samples from a joint distribution
$\D$ on labeled examples $(\bx, y)$, where $\bx \in \R^d$ is the example
and $y \in \R$ is the corresponding label, and the goal is to compute a hypothesis that
is competitive with the best-fitting halfspace (with respect to the 0-1 loss) or ReLU
(with respect to the square loss) respectively.

A halfspace (or Linear Threshold Function) is any Boolean function $f: \R^d \to \{ \pm 1\}$ of the form
$f(\bx) = \sgn \left(\langle \bw, \bx \rangle + \theta \right)$,
where $\bw \in \R^d$ is called the weight vector and $\theta$ is called the threshold.
(The function $\sign: \R \to \{ \pm 1\}$ is defined as $\sgn(u)=1$ if $u \geq 0$ and $\sgn(u)=-1$ otherwise.)
The task of learning an unknown halfspace from samples
is one of the oldest and most well-studied problems in machine learning,
starting with the Perceptron algorithm~\cite{Rosenblatt:58} and
leading to influential techniques, including SVMs~\cite{Vapnik:98}
and AdaBoost~\cite{FreundSchapire:97}. In the realizable setting~\cite{Valiant:84},
this learning problem amounts to linear programming and can be solved
in polynomial time (see, e.g.,~\cite{MT:94}) without distributional assumptions.
In contrast, in the distribution-independent agnostic model~\cite{Haussler:92, KSS:94},
even {\em weak} learning is computationally hard~\cite{GR:06, FGK+:06short, Daniely16}.

A line of research~\cite{KKMS:08, KLS09, ABL17, Daniely15, DKS18a}
has focused on learning halfspaces in the {\em distribution-specific} agnostic PAC model,
where it is assumed that the marginal distribution on the examples is well-behaved.
In this paper, we study the important case that the marginal distribution
is the standard Gaussian. For concreteness, we formally define this problem.

\begin{problem}[Agnostically Learning Halfspaces with Gaussian Marginals] \label{prob:LTF}
Let $\mathcal{C}_{\mathrm{LTF}}$ be the class of halfspaces on $\R^d$.
Given i.i.d. samples $(\bx, y)$ from a distribution $\mathcal{D}$ on $\R^d \times \{\pm 1\}$, where
the marginal $\D_{\bx}$ on $\R^d$ is the standard Gaussian $\mathcal{N}(\vec{0}, \vec{I})$ and no assumptions are
made on the labels $y$, the goal of the learning algorithm is to output a hypothesis $h: \R^d \to \{\pm 1\}$
such that with high probability we have $\pr_{(\bx, y) \sim \D} [h(\bx) \neq y] \leq \opt+\eps$,
where $\opt  = \inf_{f \in \mathcal{C}_{\mathrm{LTF}}} \pr_{(\bx, y) \sim \D} [f(\bx) \neq y]$.
\end{problem}

The $L_1$-regression algorithm of~\cite{KKMS:08} solves Problem~\ref{prob:LTF}
with sample complexity and running time $d^{O(1/\eps^2)}$~\cite{DGJ+10:bifh, DKN10}. This algorithm is also known to succeed for all log-concave distributions and for certain discrete distributions.
A related line of work~\cite{ABL17, Daniely15, DKS18a} has given
$\poly(d/\eps)$ time algorithms with weaker guarantees, specifically
with misclassification error $C \cdot \opt+\eps$, for a universal constant $C>1$.
The fastest known algorithm with optimal error is the one from~\cite{KKMS:08}.

A Rectified Linear Unit (ReLU) is any real-valued function $f: \R^d \to \R_+$ of the form
$f(\bx) = \relu \left(\langle \bw, \bx \rangle + \theta \right)$,
where $\bw \in \R^d$ is called the weight vector and $\theta$ is called the threshold.
(The function $\relu: \R \to \R_+$ is defined as $\relu(u)= \max \{ 0, u\}$.)
ReLUs are the most commonly used activation functions in modern deep neural networks.
Finding the best-fitting ReLU with respect to square-loss
is a fundamental primitive in the theory of neural networks. A number of recent works have
studied this problem both in terms of finding efficient algorithms
and proving hardness results \citep{Mahdi17, GoelKKT17, MR18, GoelKK19, DGKKS20}.
Similarly to the case of halfspaces, the problem is efficiently solvable in the realizable
setting and computationally hard in the distribution-independent
agnostic setting~\citep{MR18}. Here we study the case of Gaussian marginals,
which we now define.

\begin{problem}[ReLU Regression with Gaussian Marginals]\label{prob:ReLU}
Let $\mathcal{C}_{\relu}$ be the class of ReLUs on $\R^d$.
Given i.i.d. samples $(\bx, y)$ from a distribution $\mathcal{D}$ on $\R^d \times \R$, where
the marginal $\D_{\bx}$ on $\R^d$ is the standard Gaussian $\mathcal{N}(\vec{0}, \vec{I})$
and no assumptions are made on the labels $y$,
the goal of the learning algorithm is to output a hypothesis $h: \R^d \to \R$
such that with high probability we have $\E_{(\bx, y) \sim \D} [(h(\bx) - y)^2] \leq \opt+\eps$,
where $\opt  = \inf_{f \in \mathcal{C}_{\relu}} \E_{(\bx, y) \sim \D}[(f(\bx) - y)^2]$.
\end{problem}

Recent work~\cite{DGKKS20} gave an algorithm for Problem~\ref{prob:ReLU} with
sample complexity and runtime $d^{\poly(1/\eps)}$. While
$\poly(d/\eps)$ time algorithms are known with weaker guarantees~\cite{GoelKK19, DGKKS20},
the fastest known algorithm with $\opt+\eps$ error is the one of~\cite{DGKKS20}.

In terms of computational hardness, prior work has given evidence that
no $\poly(d, 1/\eps)$ time algorithm exists for either problem.
Specifically,~\cite{KlivansK14} gave a reduction from the problem of learning sparse parities with noise
to Problem~\ref{prob:LTF}.
Based on the presumed computational hardness of the former problem, this reduction
implies a computational lower bound of $d^{\Omega(\log(1/\eps))}$ for Problem~\ref{prob:LTF}.
More recently,~\cite{GoelKK19} gave a qualitatively similar reduction
implying a computational lower bound of $d^{\Omega(\log(1/\eps))}$ for Problem~\ref{prob:ReLU}.
Interestingly, both of these lower bounds cannot be improved in the sense that
the corresponding hard instances can be solved in time $d^{O(\log(1/\eps))}$.

In summary, the best known algorithms for Problems~\ref{prob:LTF} and~\ref{prob:ReLU}
have running time $d^{\poly(1/\eps)}$, while the best known
computational hardness results give $d^{\Omega(\log(1/\eps))}$ lower bounds.
This raises the following natural question:
\begin{center}
{\em What is the precise complexity of Problems~\ref{prob:LTF} and~\ref{prob:ReLU}?}
\end{center}
Given the lower bounds of~\cite{KlivansK14, GoelKK19}, it is conceivable that there exist
algorithms for these problems running in time $d^{\polylog(1/\eps)}$, i.e., quasi-polynomial
in $1/\eps$. A positive result of this form would represent significant algorithmic progress
in the theory of PAC learning.

{\em In this paper, we show that the latter possibility is unlikely.}
Specifically, we prove Statistical Query (SQ) lower bounds of $d^{\poly(1/\eps)}$
for both Problems~\ref{prob:LTF} and~\ref{prob:ReLU}. Our SQ lower bounds provide
evidence that known algorithms for these problems are essentially best possible.

Before we state our contributions in detail, we give some background on
Statistical Query (SQ) algorithms. SQ algorithms are a broad class of algorithms
that  are only allowed to query expectations of bounded functions of the distribution
rather than directly access samples.
The SQ model was introduced by Kearns~\cite{Kearns:98} in the context of supervised learning
as a natural restriction of the PAC model~\cite{Valiant:84}. Subsequently, the SQ model
has been extensively studied in a plethora of contexts (see, e.g.,~\cite{Feldman16b} and references therein).

Formally, an SQ algorithm has access to the following oracle.

\begin{definition}[\textsc{STAT} Oracle] \label{def:stat-oracle}
Let $\D$ be a distribution over some domain $X$ and let $f: X \to [-1, 1]$.
A statistical query is a function $q: X \times [-1, 1] \to [-1, 1]$.  We define
\textsc{STAT}$(\tau)$ to be the oracle that given a query $q(\cdot, \cdot)$
outputs a value $v$ such that
$|v - \E_{\vec x \sim \D}\left[q(\vec x, f(\vec x))\right]| \leq \tau$,
where $\tau>0$ is the tolerance parameter of the query.
\end{definition}

We note that the class of SQ algorithms is rather general and captures most of the known
supervised learning algorithms. More broadly,
a wide range of known algorithmic techniques in machine learning
are known to be implementable using SQs. These include spectral techniques,
moment and tensor methods, local search (e.g., Expectation Maximization),
and many others (see, e.g.,~\cite{Feldman13, FeldmanGV17}).
For the supervised learning problems studied in this paper, all known
algorithms with non-trivial performance guarantees are SQ or are easily
implementable using SQs.

One can prove lower bounds on the complexity
of SQ algorithms via the notion of {\em Statistical Query (SQ) dimension}~\cite{BFJ+:94, Feldman13}.
A lower bound on the SQ dimension of a learning problem provides
an unconditional lower bound on the complexity of any SQ algorithm for the problem.

\subsection{Our Results and Techniques} \label{ssec:results}
We are now ready to formally state our main results.
For Problem~\ref{prob:LTF} we prove:

\begin{theorem}\label{thm:LTF}
Let $d \geq 1$ and $\eps \geq d^{-c}$, for some sufficiently small constant $c>0$.
Any SQ algorithm that agnostically learns halfspaces on $\R^d$
under Gaussian marginals within additive error $\eps>0$
requires at least $d^{c/\eps}$ many statistical queries to $\textsc{STAT}(d^{-c/\eps})$.
\end{theorem}

Intuitively, the above statement says that any SQ algorithm for Problem~\ref{prob:LTF}
requires time at least $d^{\Omega(1/\eps)}$. This comes close to the known
upper bound of $d^{O(1/\eps^2)}$~\cite{KKMS:08} and exponentially improves
on the best known lower bound of $d^{\Omega(\log(1/\eps))}$~\cite{KlivansK14}.

For Problem~\ref{prob:ReLU} we prove:

\begin{theorem}\label{thm:ReLU}
There exist constants $c, c' >0$ such that the following holds:
For $d \geq 1$ and $\eps \geq d^{-c}$,
any SQ algorithm with excess error $\eps$  for ReLU regression on $\R^d$
under Gaussian marginals requires at least $d^{c/\eps^{c'}}$
many statistical queries to $\textsc{STAT}(d^{-c/\eps^{c'}})$.
\end{theorem}

Intuitively, the above statement says that any SQ algorithm for Problem~\ref{prob:ReLU}
requires time at least $d^{(1/\eps)^{\Omega(1)}}$. This qualitatively matches the
upper bound of $d^{\poly(1/\eps)}$~\cite{DGKKS20}, up to the degree of the polynomial,
and exponentially improves on the best known lower bound of $d^{\Omega(\log(1/\eps))}$~\cite{GoelKK19}.

\paragraph{Discussion.}
The reduction-based hardness of~\cite{KlivansK14, GoelKK19} imply
SQ lower bounds of $d^{\Omega(\log(1/\eps))}$ for both problems.
Our new SQ lower bounds are qualitatively optimal, nearly matching current algorithms.
Interestingly, for both problems, our results show a sharp separation
in the complexity of obtaining error $O(\opt)+\eps$ (which is $\poly(d/\eps)$)
versus optimal error $\opt+\eps$. In particular, our lower bounds suggest that
the accuracy-runtime tradeoff of known polynomial time approximation schemes (PTAS)
for these problems~\cite{Daniely15, DGKKS20} that achieve error $(1+\gamma)\opt+\eps$, for all $\gamma>0$,
in time $\poly(d^{\poly(1/\gamma)}, 1/\eps)$ is qualitatively best possible.

\paragraph{Technical Overview}
The starting point for our SQ lower bound
constructions is the framework of~\cite{DKS17-sq}. This work
establishes the following: Suppose we
have a one-dimensional distribution $A$ that matches its first $k$ moments
with the standard one-dimensional Gaussian. Consider the set of distributions $\{ \mathbf{P}_{\bv} \}$,
where $\bv$ is any unit vector, such that the projection of $\mathbf{P}_{\bv}$ in the $\bv$-direction
is equal to $A$ and in the orthogonal complement $\mathbf{P}_{\bv}$ is an independent standard Gaussian.
Then this set of distributions has SQ dimension $d^{\Omega(k)}$.
By known results (see, e.g.,~\cite{Feldman16b}) this implies that
distinguishing such a distribution from the standard Gaussian
or learning a distribution with better than $1/\poly(d^k)$ correlation
with such a distribution is hard in the SQ model.

To leverage the aforementioned result, a first hurdle that must be overcome
is adapting the results of~\cite{DKS17-sq} -- which apply to the unsupervised task of
learning distributions -- to the supervised task of learning functions. For Boolean
functions $F: \R^d \to \{ \pm 1\}$, this is relatively straightforward.
Essentially, sampling from the distribution $(\bx,F(\bx))$ is equivalent to
sampling from the conditional distributions of $\bx$ conditioned on $F(\bx)=1$
and on $F(\bx)=-1$ in a way that is easy to make rigorous in the SQ model.
To apply the techniques from~\cite{DKS17-sq}, we want to construct
a one-dimensional Boolean-valued function $f :\R \to \{ \pm 1\}$ such that
the conditional distributions match moments, or more conveniently so that
$\E_{z \sim \mathcal{N}(0, 1)}[f(z) \, z^i] = 0$ for all $0 \leq i \leq k$.
Given such a moment matching function, it is not hard to show that
the conditional distributions corresponding to the function
$F_{\bv}(\bx) = f(\langle \bx, \bv \rangle)$,
for a unit vector $\bv$, are of the type covered by~\cite{DKS17-sq}.

Of course, for such a construction to have any implications for the problems
of agnostically learning halfspaces/ReLUs, we require an additional key property:
We want to find a univariate Boolean-valued function $f$
that not only has this kind of matching moments property, but also
{\em correlates non-trivially} with a halfspace or ReLU. If we can guarantee
non-trivial correlation, agnostically learning the function $F_{\bv}$
with respect to this class will require having a {\em weak learner} for $F$
(and, for example, preventing the learning from just outputting the identically zero function).

To achieve both aforementioned goals, we use analytic tools
to show that there exists an $O(k)$-piecewise constant Boolean-valued
function $f$ with $k$ matching moments. Such a function has correlation
$\Omega(1/k)$ with some halfspace.
For the case or ReLUs, we use an analysis making use of Legendre polynomials to find a function
$f' : \R \to [-1,1]$ with vanishing first $k$ moments and non-trivial correlation with a ReLU.
We then show that this can be rounded to a Boolean-valued function with the same guarantees.

\paragraph{Concurrent Work} Concurrent and independent work~\cite{GGK20}
established qualitatively similar SQ lower bounds for agnostically learning halfspaces and ReLUs.
Their techniques are different than ours, building on prior SQ lower bounds for 
learning depth-$2$ neural networks~\cite{DKKZ20, GGJKK20}, and employing a reduction-based approach.

 \section{Preliminaries} \label{sec:prelims}

\paragraph {\bf Notation}
For $n \in \Z_+$, we denote $[n] \eqdef \{1, \ldots, n\}$ 
and $\mathbb{\overline R} \eqdef \R \cup\{\pm \infty\}$. 
For $\bx \in \R^d$, and $i \in [d]$,
$\bx_i$ denotes the $i$-th coordinate of $\bx$.
We will use $\langle \bx, \by \rangle$ for the inner product between $\bx, \by
\in \R^d$.  We will use $\E[X]$ for the expectation of random variable $X$
and $\pr[\mathcal{E}]$ for the probability of event $\mathcal{E}$.

Let $\vec e_i$ be the $i$-th standard basis vector in $\R^d$. 
Let $\mathcal{N}(0,1)$ denote the standard univariate Gaussian distribution 
and $\mathcal{N}(\vec 0,\vec I)$ denote the standard multivariate Gaussian distribution.
We will use $\phi$ to denote the pdf of the standard Gaussian. 

\paragraph{Correlation and Statistical Query Dimension}
To bound the complexity of SQ learning a concept class $\cal C$, 
we will use the standard notion of Statistical Query Dimension~\cite{BFJ+:94}. 

For $f,g: \R^d \to \R$, we define the correlation between $f$ and $g$ 
under the distribution $\D$ to be the expectation $\E_{\bx \sim \D}[f(\x)g(\x)]$.
To prove that the SQ dimension of $\cal C$ under the distribution $\D$ is large, 
we need to find a set of functions in the class that are nearly uncorrelated.

\begin{definition}[Statistical Query Dimension]
For a class of functions $\cal C$ and distribution $\D$, 
$\textsc{SQ-DIM}({\cal C},\D)=s$, if $s$ is the largest integer value 
for which there exist $s$ functions $f_1, f_2, \ldots, f_s \in\cal C$ 
such that for every $i\neq j$, it holds $|\E_{\vec x \sim \D}[f_i(\x)f_j(\x)]|\leq 1/s$.
\end{definition}

Our SQ lower bounds will use the following lemma (see, e.g., Theorem 2 of \cite{szorenyi2009characterizing}).

\begin{lemma}[]\label{lem:fieldman_lemma}
Let $\cal C$ be a concept class of functions on $\R^d$ and $\D$ be a distribution on $\R^d$. 
Let $s=\textsc{SQ-DIM}({\cal C},\D)$. Any SQ algorithm that outputs a hypothesis 
with correlation at least $1/s^{1/3}$ from an unknown function in $\cal C$
requires at least $s^{1/3}/2-1$ queries to $\textsc{STAT}(1/s^{1/3})$.
\end{lemma}

We note that the above theorem was initially shown for Boolean-valued functions,
but also holds for real-valued functions of bounded norm (see, e.g., \cite{Feldman:09}). \section{SQ Lower Bound for Agnostically Learning Halfspaces} \label{lb:ltfs}

In this section, we prove Theorem~\ref{thm:LTF}.
To do so, we construct a family $\mathcal{F}_k$ of $k$-decision
lists of halfspaces on $\R^d$
satisfying the following properties:
(1) Any SQ algorithm that weakly learns $\mathcal{F}_k$
requires many high accuracy SQ queries.
(2) Each $F \in \mathcal{F}_k$ is non-trivially correlated with a halfspace.
Formally, we establish the following statement.

\begin{proposition}\label{prop:LTF}
Assuming $d$ is at least a sufficiently large power of $k$, there exists a set $\mathcal{F}_k$ of $k$-decision lists
of halfspaces on $\R^d$ such that any SQ algorithm that learns $\mathcal{F}_k$
within 0-1 error $\leq 1/2 - d^{-\Omega(k)}$ with respect to $\normal(\vec 0,\vec I)$ 
needs $d^{\Omega(k)}$ queries to $\textsc{STAT}(d^{-\Omega(k)})$. Moreover, 
for any $F \in \mathcal{F}_k$, there is a halfspace $\sigma$
such that $\E_{\x \sim \normal(\vec 0,\vec I)}[F(\bx) \sigma(\bx)]\geq 1/(2k)$.
\end{proposition}

Given the above statement, Theorem~\ref{thm:LTF} follows.

\begin{proof}[Proof of Theorem~\ref{thm:LTF}]
Let $\mathcal{A}$ be an agnostic SQ learner for halfspaces
under Gaussian marginals. We use $\mathcal{A}$ to weakly learn
$\mathcal{F}_k$, for a value of $k$ to be determined.
That is, we feed $\mathcal{A}$ a set of i.i.d. labeled examples
from an arbitrary function $F \in \mathcal{F}_k$.
By definition, algorithm $\mathcal{A}$ computes a hypothesis
$h: \R^d \to \{ \pm 1\}$ such that
$\pr_{\x \sim \normal(\vec 0,\vec I)}[h(\x)\neq F(\x)] \leq \opt+\eps$, for $\eps>0$.
By the second statement of Proposition~\ref{prop:LTF}, it follows that $\opt \leq 1/2-1/(2k)$.
Thus, we have that
\[ \pr_{\x\sim \normal(\vec 0,\vec I)}[h(\x)\neq F(\x)]\leq 1/2-1/(2k) +\eps\;.\]
For $\eps = 1/(4k)$, Proposition~\ref{prop:LTF} gives that $\cal A$ needs 
at least $d^{\Omega(1/\eps)}$ queries to $\textsc{STAT}(d^{-\Omega(1/\eps)})$. 
This completes the proof.
\end{proof}

\subsection{Proof of Proposition~\ref{prop:LTF}} \label{ssec:main-prop-ltfs}

The main idea for our construction of a hard-to-learn family of functions $\mathcal{F}_k$
is the following: We first establish the existence of a one-dimensional
Boolean-valued function $f: \R \to \{\pm1 \}$ whose first $k$ moments match
the first $k$ moments of the standard univariate Gaussian distribution (Proposition~\ref{prop:main_structural}).
Importantly, this function $f$ is $(k+1)$-piecewise constant, i.e.,
there exists a partition of its domain into $k+1$ intervals $I_1, \ldots, I_k$
such that $f$ is constant within each $I_j$. The $k$ points
$z_1, z_2, \ldots z_{k}\in \R$ where the function changes value
are called breakpoints.

Given our univariate construction, we construct our family of $d$-dimensional functions
by using a copy of this one-dimensional function $f$ oriented in a random direction.
More specifically, let $S$ be a set of $2^{d^{\Omega(1)}}$
nearly orthogonal unit vectors on $\R^d$ (Lemma~\ref{lem:near-orth-vec}).
Then we define the family ${\cal F}_k =\{F_{\vec v}(\x)\}_{\vec v \in S}$,
where $F_{\vec v}(\x)\eqdef f(\dotp{\vec v}{\x})$ for
the univariate function $f$ from Proposition~\ref{prop:main_structural}.
Since $f$ is $k$-piecewise constant, each $F \in \mathcal{F}_k$
is a $k$-decision list. From this, it follows that each such $F$ is non-trivially
correlated with a halfspace (Lemma~\ref{lem:corel_ltf}).

Figure~\ref{fig:fv} shows how the function $F_{\vec v}$
labels the examples in a $2$-dimensional projection.

\begin{figure}[ht]
	\centering
	\begin{tikzpicture}[thick,rotate=45]
	\pgfmathsetmacro{\radius}{2}
	\draw[clip] (0,0) circle (\radius);
	\fill[red, fill opacity=0.20] (-1,\radius) rectangle (-\radius,-\radius);
	\fill[blue, fill opacity=0.20] (-1,\radius) rectangle (-0.5,-\radius);
	\fill[red, fill opacity=0.20]  (-0.5,-\radius) rectangle (0.2,\radius);
	\fill[blue, fill opacity=0.20] (0.2,\radius) rectangle (1,-\radius);
	\fill[red, fill opacity=0.20] (1,-\radius) rectangle (1.3,\radius);
	\fill[blue, fill opacity=0.20] (1.3,\radius) rectangle (1.5,-\radius);
	\fill[red, fill opacity=0.20] (1.5,-\radius) rectangle (1.9,\radius);
\fill[blue, fill opacity=0.20] (1.9,\radius) rectangle (2,-\radius);
\draw[->] (-\radius,-\radius) -- (\radius,\radius);
	\draw[->] (\radius,-\radius) -- (-\radius,\radius);
	\draw[->] (0,0) -- (1,0)node[below,black] {$\vec v$};
	\draw[dashed] (-\radius,0) --(\radius,0);
	\end{tikzpicture}
	\caption{The ``red'' region is the set of points where $F_{\vec v}(\bx) = -1$
		and the ``blue'' region where $F_{\vec v}(\bx) = 1$.}
\label{fig:fv}
\end{figure}

We show (Proposition~\ref{prop:LTF}) that any SQ algorithm
that can distinguish between an unknown $F_{\vec v}$ and
a function with uniformly random  $\pm1$ labels requires many high accuracy
queries.

The key structural result that we require is the following:

\begin{proposition}\label{prop:main_structural}
For any $k \geq 1$, there exists an at most $(k+1)$-piecewise constant function $f: \R \to \{\pm 1\}$
such that $\E_{z \sim \normal(0,1)}[f(z)z^t]=0$, for every non-negative integer $t<k$.
\end{proposition}

Proposition~\ref{prop:main_structural} is the most technically involved result of this section.
We give its proof in Section~\ref{ssec:prop-ltfs}.

In the remainder of this subsection, we prove Proposition~\ref{prop:LTF},
assuming Proposition~\ref{prop:main_structural}.

We say that a distribution $A$ has $k$-matching moments with $\normal(0,1)$
if $\E_{z \sim A}[z^t]=\E_{z \sim \normal(0, 1)}[z^t]$, for all $0\leq t<k$.
Proposition~\ref{prop:main_structural} implies the following.

\begin{lemma}\label{lem:matching_moments}
Let $f: \R \to \{\pm 1\}$ be such that $\E_{z \sim \normal(0,1)}[f(z)z^t]=0$,
for every non-negative integer $t<k$. For $z\sim \normal(0,1)$, define $A \eqdef \1 \{f(z)=1\}$
and $B \eqdef \1 \{f(z)=-1\}$. Then the random variables $A$ and $B$ have $k$-matching moments with $z$.
\end{lemma}
\begin{proof}
We prove the lemma for the random variable $A$. 
The proof for $B$ is similar. 
From the definition of $f$, we have that 
$\E_{z \sim \normal(0,1)}[f(z)]=0$, thus 
\[\E_{z \sim \normal(0,1)}[ \1\{f(z)=1\}]=\E_{z \sim \normal(0,1)}[\1\{f(z)=-1\}] \]
or equivalently 
\begin{equation}\label{eq:value_of_f}
\E_{z \sim \normal(0,1)}[ \1\{f(z)=1\}]=\frac 12\;.
\end{equation}
Similarly, from $\E_{z \sim \normal(0,1)}[f(z)z^t]=0$, we have
\begin{equation}\label{eq:moments}
\E_{z \sim \normal(0,1)}[z^t\1\{f(z)=1\}]=	\E_{z \sim \normal(0,1)}[z^t \1\{f(z)=-1\}]\;.
\end{equation}
Let $ \phi(z| f(z)=1)$ be the probability distribution of $z$ conditional that $f(z)=1$. We have that
\begin{align*}
\E_{z \sim A}[z^t]&=\int_{-\infty}^{\infty} z^t \phi(z| f(z)=1) \d z=
\int_{-\infty}^{\infty} z^t \frac{\phi(z)}{\pr_{z'\sim\normal(0,1)}[f(z')=1]} \1\{f(z)=1\} \d z\\
&= 2 \int_{-\infty}^{\infty} z^t \phi(z) \1\{f(z)=1\} \d z=
\int_{-\infty}^{\infty} z^t \phi(z) \1\{f(z)=1\} \d z+ \int_{-\infty}^{\infty} z^t
\phi(z) \1\{f(z)=-1\} \d z\\&= \int_{-\infty}^{\infty} z^t \phi(z) \d z=	\E_{z \sim \normal(0,1)}[z^t]\;,
\end{align*}
where we used Equations~\eqref{eq:value_of_f}, \eqref{eq:moments}.
\end{proof}
We will require the following technical lemmas from~\cite{DKS17-sq}.
The first lemma says that there exists a large set of unit vectors
that are pairwise nearly orthogonal.

\begin{lemma}[Lemma 3.7 of \cite{DKS17-sq}]\label{lem:near-orth-vec}
For any $0<c<1/2$, there exists a set $S$ of $2^{\Omega(d^c)}$ unit vectors in $\R^d$
such that for each pair of distinct $\vec u,\vec v\in S$, we have
$|\dotp{\vec u}{\vec v}| \leq O(d^{c-1/2})$.
\end{lemma}

If we define the conditional distribution on the event $F_{\vec v}(\x)=1$,
we can see that the directions orthogonal to $\vec v$ follow a standard $(d-1)$-dimensional Gaussian distribution.
Thus, we define the following distribution for this case.

\begin{definition}[High Dimensional Hidden Direction Distribution]
For a univariate distribution $A$ with probability density function $A(z)$
and a unit vector $\vec u\in \R^d$, consider the distribution over $\R^d$ with pdf
$\mathbf{P}_{\vec u}(\x)= A(\dotp{\vec u}{\x})\exp(-\snorm{2}{\x-\dotp{\vec u}{\x} \vec u}^2/2)/(2\pi)^{(d-1)/2}$.
That is, $	\mathbf{P}_{\vec u}$ is the product distribution whose orthogonal projection
onto the direction $\vec u$ is $A$, and onto the subspace perpendicular to $\vec u$
is the standard $(d-1)$-dimensional Gaussian distribution.
\end{definition}

Let $\D_1,\D_2: \R^d \to \R_+$ be probability density functions.
\nnew{The $\chi^2$-divergence of $\D_1,\D_2$ is defined as
$\chi^2(\D_1,\D_2)\eqdef \int_{\R^d} \D_1(\x)^2/\D_2(\x) \d \x-1$.}
We \nnew{also} define the correlation between $\D_1,\D_2$ and a reference distribution $\D$
as $\chi_\D (\D_1,\D_2)\eqdef \int_{\R^d} \D_1(\x)\D_2(\x)/\D(\x) \d \x-1$.

The second lemma states that the distributions $\mathbf{P}_{\vec u}, \mathbf{P}_{\vec v}$
have correlation depending on the angle of the corresponding vectors.

\begin{lemma}[Lemma 3.4 of \cite{DKS17-sq}] \label{lem:corelation_lem}
For any $\vec u, \vec v \in \Sp^{d-1}$, let $A: \R \to \R_+$ be the pdf of a
distribution that agrees with the first $k$ moments of $\normal(0,1)$. Then, we have that
\[ |\chi_{\normal(\vec 0,\vec I)}(\mathbf{P}_{\vec u},\mathbf{P}_{\vec v})|\leq |\dotp{\vec u}{\vec
		v}|^{k+1} \chi^2 (A,\normal(0,1))\;.\]
\end{lemma}

The second statement of Proposition~\ref{prop:LTF}, establishing non-trivial correlation
with a halfspace is shown in the lemma below.

\begin{lemma}\label{lem:corel_ltf}
For any $F \in \mathcal{F}_k$, there is a halfspace $\sigma$
such that $\E_{\x \sim \normal(\vec 0,\vec I)}[F(\bx) \sigma(\bx)]\geq 1/(2k)$.
\end{lemma}
\begin{proof}
We start by noting that each $F \in \mathcal{F}_k$ is of the form
$F_{\vec v}(\vec x)=f(\dotp{\vec v}{\vec x})$, where $f$ is the
function from Proposition~\ref{prop:main_structural} and 
$\vec v\in S$. We will take $\sigma$ to be $\sigma(\x)= \sign(\dotp{\vec v}{\vec x} + \beta)$. 
Let $z_1, \ldots, z_k$ be the breakpoints of $f(z)$. We will show that if we set the value of $\beta$
to a breakpoint, then the result follows. 

Let $a_{i+1}=\int_{z_i}^{z_{i+1}}  \phi(z)\d z$ for $0<i<k+1$, 
$a_{1}=\int_{-\infty}^{z_1}  \phi(z) \d z$ and 
$a_{k+1}=\int_{z_k}^{\infty}  \phi(z)\d z$.  Let $\beta=z_l$, for a breakpoint $z_l$, and $b=\sign(f((z_l + z_{l+1})/2))$. 
Then we have that
\[ 
\E_{\x \sim \normal(\vec 0,\vec I)}[F_{\vec v}(\bx) \sigma(\bx)] = 
2\int_{z_l}^{\infty} f(z) \phi(z)\d z= 2b \sum_{j=l}^{k+1} (-1)^{j-l} a_j\;,
\]
where the first equality holds because 
\[ 
\E_{\x \sim \normal(\vec 0,\vec I)}[F_{\vec v}(\bx)\1\{\x\in A\}] = 
- \E_{\x \sim \normal(\vec 0,\vec I)}[F_{\vec v}(\bx)\1\{\x\in A^c\}]  \;, 
\] 
for any $A\subseteq\R^d$.  
From the fact that $\sum_{i=1}^{k+1} a_i=1$, it follows that there exists an index $i$ such
that $a_i\geq 1/(k+1)$. Assume, for the sake of contradiction, that for all $l>i$ we have that
$|\sum_{j=l}^{k+1}(-1)^{l-j}a_j| \leq (1/4k)$, since otherwise there exists a breakpoint that satisfies the equation. 
Then, for $b=\sign(f((z_i +z_{i+1})/2))$,
we have that either $2 b (\sum_{j=i+1}^{k+1}(-1)^{j-i}a_j + a_i)\geq (1/2k)$ or
$2 b(\sum_{j=i+1}^{k+1}(-1)^{j-i}a_j + a_i)\leq -(1/2k)$. In the former case, we are done.
In the latter case, the halfspace $-\sigma(\x)$ satisfies the desired correlation property. 
\end{proof}

We are now ready to prove (the first statement of) Proposition~\ref{prop:LTF}.

\begin{proof}[Proof of Proposition~\ref{prop:LTF}]
Let $S$ be the set of nearly orthogonal vectors from Lemma~\ref{lem:near-orth-vec}
and $Y$ be a uniform $\pm 1$ random variable. 
For each $\vec v\in S$, let $F_{\vec v}(\vec x)=f(\dotp{\vec v}{\vec x})$, where $f$ is the
function from Proposition~\ref{prop:main_structural}. Let $\D_D$ be the set that
contains the distributions $(X,F_{\vec v}(X))$ for any $\vec v\in S$ and $X\sim \normal(\vec 0,\vec I)$. 
We will prove that for any $\vec u,\vec v\in S$, $\vec u\neq \vec v$, we have
\begin{equation}\label{eq:basic_corelation}
\chi_{(X,Y)}\left( (X,F_{\vec u}(X)) ,(X,F_{\vec v}(X)) \right) \leq 2\cdot|\dotp{\vec v}{\vec u}| ^{k+1}\;.
\end{equation}
To prove Equation~\eqref{eq:basic_corelation}, for a unit vector $\vec v$,
denote by $A_{\vec v}$ the conditional distribution of the event $F_{\vec v}=1$,
and by $B_{\vec v}$ the conditional distribution of the event $F_{\vec v}=-1$.
Let $\D_{\vec v}$ be the probability density function of $(X,F_{\vec v}(X))$.
We then have
\begin{align}
\chi_{(X,Y)}\left( (X,F_{\vec u}) ,(X,F_{\vec v}) \right)&= 2\int_{\R^d}
\frac{\D_{\vec u} (\x,1) \D_{\vec v} (\x,1)}{\phi(\x)} \d \x + 2\int_{\R^d}
\frac{\D_{\vec u} (\x,-1) \D_{\vec v} (\x,-1)}{\phi(\x)} \d \x -1 \nonumber\\
&= \frac 12 \left(	\chi_{X}\left( A_{\vec u},A_{\vec v} \right) +
\chi_{X}\left( B_{\vec u},B_{\vec v} \right) \right)
\label{eq:main_proof_eq1}\;,
\end{align}
where we used that $Y$ gets each label with probability $1/2$ and that
$\pr_{\x\sim \normal(\vec 0,\vec I)}[F_{\vec v}(\vec x)=\pm1]=1/2$.
From Lemma~\ref{lem:corelation_lem} and Lemma~\ref{lem:matching_moments}, we
have
\begin{align}
\left(	\chi_{X}\left( A_{\vec u},A_{\vec v} \right) +	\chi_{X}\left(B_{\vec u},B_{\vec v} \right) \right)&\leq  \dotp{\vec u}{\vec v}^{k+1}
\left(\chi^2(A,\normal(0,1)) + \chi^2(B,\normal(0,1))\right)\label{eq:main_proof_eq2}\;.
\end{align}
We also have
\begin{align}
\chi^2(A,\normal(0,1)) &= \int_{-\infty}^{\infty} A(z)^2/\phi(z)\d z =
\int_{-\infty}^{\infty} \frac{\phi(z)^2}{\phi(z) \pr_{z'\sim \normal(0,1)}[f(z')=1]^2} \1\{f(z)=1\}\d z\nonumber\\
&=4 \int_{-\infty}^{\infty} \phi(z) \1\{f(z)=1\}\d z =2 \label{eq:main_proof_eq3}\;,
\end{align}
where we used the conditional expectation and Equation~\eqref{eq:value_of_f}. Putting
Equations~\eqref{eq:main_proof_eq1}, \eqref{eq:main_proof_eq2} and
\eqref{eq:main_proof_eq3} together, we get Equation~\eqref{eq:basic_corelation}.
Using Lemma~\ref{lem:near-orth-vec}, we have that $|\dotp{\vec v}{\vec u}|\leq d^{-(1/2-c)}$, thus
\begin{equation*}\label{eq:correlation_gamma}
\chi_{(X,Y)}\left( (X,F_{\vec u}(X)) ,(X,F_{\vec v}(X)) \right) \leq
\Omega\left( d^{-(k+1)(1/2-c)}\right)\;.
\end{equation*}
To finish our argument, for $\vec u\neq \vec v$, we have that
\[
\chi_{(X,Y)}\left( (X,F_{\vec u}(X)) ,(X,F_{\vec v}(X)) \right)= \E_{X \sim \normal(\vec 0,\vec I)}[F_{\vec u}(X) F_{\vec v}(X)] 
\leq \Omega(d^{-(k+1)(1/2-c)})\;.\]
Thus, we have that $\textsc{SQ-DIM}({\cal F}_k,\normal(\vec 0,\vec I))=\min(d^{\Omega(k)},2^{d^c})=d^{\Omega(k)}$,
where the last inequality uses the relation between $k$ and $d$.
By Lemma~\ref{lem:fieldman_lemma}, any SQ algorithm that finds a function $h$ such
that $\pr_{\x\sim \normal(\vec 0, \vec I)}[F_{\vec v}(\x)\neq h(\x)] \leq 1/2 -d^{-\Omega(k)}$
needs at least $d^{\Omega(k)}$ queries to $\textsc{STAT}(d^{-\Omega(k)})$.
This completes the proof.
\end{proof}

\subsection{Proof of Proposition~\ref{prop:main_structural}} \label{ssec:prop-ltfs}

The key lemma for the proof is the following.

\begin{lemma}\label{lem:main_diff}
Let $m$ and $k$ be positive integers such that $m>k+1$ and $\eps>0$.
If there exists an $m$-piecewise constant $f:\R \mapsto \{\pm 1\}$
such that $|\E_{z \sim \normal(0,1)}[f(z)z^t]|<\eps$ for all non-negative integers $t<k$,
then there exists an at most $(m-1)$-piecewise constant $g :\R \mapsto \{\pm 1\}$
such that $|\E_{z \sim \normal(0,1)}[g(z)z^t]|<\eps$ for all non-negative integers $t<k$.
\end{lemma}

\begin{proof}
Let $\{b_1,b_2,\ldots,b_{m-1}\}$ be the breakpoints of $f$.
	Then let $F(z_1, z_2, \ldots, z_{m-1},z):\mathbb{\overline R}^m \mapsto \R$ be an $m$-piecewise constant function
	with breakpoints on $z_1, \ldots, z_{m-1}$, where $z_1<z_2< \ldots <z_{m-1}$
	and $F(b_1, b_2, \ldots, b_{m-1},z)=f(z)$.  For simplicity, let $\vec z=(z_1, \ldots, z_{m-1})$
	and  define $M_i(\vec z)= \E_{z \sim \normal(0,1)}[F(\vec z, z)z^i]$ and let
	$\vec M(\vec z)=[M_0(\vec z), M_1(\vec z), \ldots M_{k-1}(\vec z)]^T$. It is
	clear from the definition that
	$M_i(\vec z)=\sum_{n=0}^{m-1}\int_{z_n}^{z_{n+1}} F(\vec z, z) z^i \phi(z) \d z =
	\sum_{n=0}^{m-1}a_n\int_{z_n}^{z_{n+1}} z^i \phi(z) \d z$,
	where $z_0= -\infty$ and $z_m=\infty$ and $a_n$ is the sign of $F(\vec z,z)$ in the interval $(z_n,z_{n+1})$.
	Note that $a_n=-a_{n+1}$ for every $0\leq n<m$.
	By taking the derivative of $M_i$ in $z_j$, for $0<j<m$, we get that
	\[\frac{\partial}{\partial z_j} M_i(\vec z)= 2a_{j-1} z_j^i \phi(z_j) \quad \text{and}\quad
	\frac{\partial}{\partial z_j} \vec M(\vec z)= 2a_{j-1} \phi(z_j) [1, z_j^1, \ldots ,z _j^{k-1}]^T\;.\]
	We now argue that for any $\vec z$ with distinct coordinates that there exists a vector $\vec u\in \R^{m-1}$ such that
	$\vec u=(\vec u_1,\ldots,\vec u_k,0,0,\ldots,0,1)$ and the directional derivative of $\vec M$ in the $\vec u$ direction
	is zero. To prove this, we construct a system of linear equations such that
	$\nabla_{\vec u} M_i(\vec z)=0$, for all $0\leq i<k$. Indeed, we have
$\sum_{j=1}^{k} \frac{\partial}{\partial z_j}  M_i(\vec z) \vec u_j
	= - \frac{\partial}{\partial z_{m-1}}  M_i(\vec z) $ or $\sum_{j=1}^{k} a_{j-1} z_j^i \phi(z_j)\vec u_j=- a_{m-2} z_{m-1}^i \phi(z_{m-1})$,
	 which is linear in the variables $\vec u_j$. Let $\hat{\vec u}$ be the vector with the first $k$ variables 
	 and let $\vec w$ be the vector of the right hand side of the system, i.e., $\vec w_i=- a_{m-2} z_{m-1}^i \phi(z_{m-1})$. Then
	 this system can be written in matrix form as $\vec V \vec D\hat{ \vec u}=\vec w$, where $\vec V$ is the Vandermonde matrix,
	 i.e., the matrix that is $\vec V_{i,j}=\alpha_i^{j-1}$, for some values $\alpha_i$ and $\vec D$ is a diagonal matrix.
	 In our case, $\vec V_{i,j}=z_i^{j-1}$ and $\vec D_{j,j}= 2 a_{j-1}\phi(z_j)$.
	 It is known that the Vandermonde matrix has full rank iff for all $i\neq j$ we have $\alpha_i\neq \alpha_j$,
	 which holds in our setting. Thus, the matrix $\vec V \vec D$ is nonsingular and there exists a solution to the equation.
Thus, there exists a vector $\vec u$ with our desired properties and, moreover,
	any vector in this direction is a solution of this system of linear equations.
	Note that the vector $\vec u$ depends on the value of $\vec z$,
	thus we consider $\vec u(\vec z)$ be the (continuous) function that returns a vector $\vec u$ given $\vec z$.
	
We define a differential equation for the function $\vec v:\mathbb{\overline R}\mapsto\mathbb{\overline R}^{m-1}$, as follows: $\vec v(0)= \vec b$, where $\vec b=(b_1, \ldots, b_{m-1})$, and
$ \vec v'(T)=\vec u(\vec v(T))$ for all $T \in \mathbb{\overline R}$.
If $\vec v$ is a solution to this differential equation, then we have:
\[\frac{\d}{\d T} \vec M(\vec v(T))=\frac{\d}{\d \vec v(T)} \vec M(\vec v(T)) \frac{\d}{\d T} \vec v(T)
=\frac{\d}{\d \vec v(T)} \vec M(\vec v(T)) \vec u(\vec v(T)) =\vec 0\;,
\]
where we used the chain rule and that the directional derivative in $\vec u(\vec v(T))$ direction is zero.
This means that the function $\vec M(\vec v(t))$ is constant, and for all $0\leq j<k$, we have $|M_j|< \eps$, because we have that $|\E_{z \sim \normal(0,1)}[F(z_1,\ldots, z_{m-1},z)z^t]|<\eps$. Furthermore, since $\vec u(\vec v(T))$ is continuous in $\vec v(T)$, this differential equation will be well founded and have a solution up until the point where either two of the $z_i$ approach each other or one of the $z_i$ approaches plus or minus infinity (the solution cannot oscillate, since $\vec v_{m-1}'(T)=1$ for all $T$).
	
Running the differential equation until we reach such a limit, we find a limiting value $\vec v^\ast$ of $\vec v(T)$ so that either:
\begin{enumerate}[leftmargin=*]
\item There is an $i$ such that $\vec v_i^\ast=\vec v_{i+1}^\ast$, which
gives us a function that is at most $(m-2)$-piecewise constant, i.e., taking $F(\vec v^\ast,z)$.
\item Either $\vec v_{m-1}^\ast = \infty$ or $\vec v_1^\ast = -\infty$, which gives us an at most
$(m-1)$-piecewise constant function, i.e., taking $F(\vec v^\ast,z)$.
Since when the $\vec v_{m-1}^\ast= \infty$, the last breakpoint becomes $\infty$, we have one less breakpoint, and if $\vec v_1^\ast =-\infty$ we lose the first breakpoint.
\end{enumerate}
Thus, in either case we have a function with at most $m-1$ breakpoints and the same moments.
This completes the proof.
\end{proof}

We also require the following simple fact, 
establishing the existence of a $k'$-piecewise
constant Boolean-valued function (for some finite $k'$), satisfying the desired moment
conditions.

\begin{fact}\label{fct:function}
For any $\eps>0$, there exists a $(k/\eps)^{O(k)}$-piecewise constant function
$f:\R\mapsto\{\pm 1\}$ such that $|\E_{z \sim \normal(0,1)}[f(z)z^t]|\leq \eps$, for every integer
$0\leq t<k$.
\end{fact}

\begin{proof}
We define $f$ to take alternative values $\pm 1$ in intervals of length $s$.
Let us denote $I_i=  (is,(i+1)s)$ for $-1/(s \eps^k) \leq i \leq 1/(s \eps^k)$ for an integer $i$. 
If $f(z)=1$ for $z\in I_i$, then we have $f(z)=-1$ for $z\in I_{i+1}$. 
Moreover, we will have that $f(z)=1$ for $z \leq -1/\eps^k$ and  $f(z)=-1$ for $z>1/\eps^k$. 
We assume that the number of constant pieces is even for simplicity.
To prove that for all $0\leq t<k$,  $\E_{z \sim \normal(0, 1)}[f(z)z^t]<4 \eps$, observe
that for all even moments the expectation is equal to zero. 
So it suffices to prove the desired statement for odd moments.
Note that $\E_{z \sim \normal(0, 1)}[z^t f(z) \1\{z\geq 0\}]= \E_{z \sim \normal(0, 1)}[|z^t| f(z) \1\{z<0\}]$
for odd moments. Thus, we will prove that $\E_{z \sim \normal(0, 1)}[z^t f(z) \1\{z\geq 0\}] \leq  2 \eps$.
We have  that 
\[ \int_{1/\eps^k}^{\infty} z^t \phi(z)\d z\leq \eps^{k/t}\;, \]
where we used the inequality 
$\pr_{z\sim \normal(0,1)}[|z|^t\geq y]\leq \frac{1}{\sqrt{2\pi}y^{1/t}} e^{-y^{2/t}}\leq 1/y^{1/t}$.	
Moreover, we bound from above the absolute ratio between two subsequent regions, 
i.e., $\left|\frac{\E_{z \sim \normal(0, 1)}[z^t f(z) \1\{z\in I_i\}]}{\E_{z \sim \normal(0, 1)}[z^t f(z) \1\{z\in I_{i+1}\}]}\right|$. 
For $i\geq 0$, we have that 
\begin{align}\label{eq:bound_ratio}
\frac{\int_{i s}^{(i+1) s} z^t \phi(z)\d z}{\int_{(i +1)s}^{(i+2) s} z^t \phi(z)\d
	z}\leq \frac{s ((i+1)s)^t \phi(i s)}{s ((i+1)s)^t \phi((i+2) s)}= e^{2i s^2  +2
	s^2}\leq 1+ 3i s^2 + 9i^2s^4 \;,
\end{align}
where in the first inequality we used the maximum value and the minimum of the
integral, and in the second one we used that  $e^x\leq 1+x+x^2$ for $x\leq 1$, which holds for $s<\eps^k$.
Thus, for two subsequent intervals we have
\[\int_{i s}^{(i+1) s} z^t \phi(z)\d z - \int_{(i +1)s}^{(i+2) s} z^t \phi(z)\d z\leq
4 i s^2\int_{(i +1)s}^{(i+2)s} z^t \phi(z)\d z \leq 4 k s^2\int_{(i +1)s}^{(i+2)
	s} z^t \phi(z)\d z \;.\]
On the other direction, from Equation~\eqref{eq:bound_ratio} we have that 
\[-\int_{i s}^{(i+1) s} z^t \phi(z)\d z + \int_{(i +1)s}^{(i+2) s} z^t \phi(z)\d
z\geq  -4 i s^2\int_{(i +1)s}^{(i+2) s} z^t \phi(z)\d z \geq - 4 k s^2\int_{(i
	+1)s}^{(i+2) s} z^t \phi(z)\d z \;. \]
Thus, we have 
\[ -4ks^2 (t-1)!! \leq \sum_{i=0}^{1/(s \eps^k)}(-1)^i\int_{i
	s}^{(i+1) s} z^t \phi(z)\d z\leq 4ks^2 \int_{-\infty}^{\infty}z^t \phi(z)=4ks^2
(t-1)!! \;.\]
Choosing $s=\eps^{(k+1)/2}/k^k$ and setting $\eps=\eps/2$, the proof follows.
\end{proof}

Proposition~\ref{prop:main_structural} follows from the above using a compactness argument.

\begin{proof}[Proof of Proposition~\ref{prop:main_structural}]
	For every $\eps>0$, using the function $f'$ from Fact~\ref{fct:function} and
	Lemma~\ref{lem:main_diff}, we can obtain a function $f_{\eps}$ such that $|\E_{z
		\sim \normal(0,1)}[f_\eps(z)z^t]|\leq\eps$, for every non-negative integer $t<k$
	and the function $f_{\eps}$ is at most $(k+1)$-piecewise constant.
	Let $\vec M: \mathbb{\overline R}^{k} \mapsto \R^{k}$, where $M_i(\vec
	b)=\sum_{n=0}^{k}(-1)^{n+1}\int_{b_n}^{b_{n+1}} z^i \phi(z) \d z$ and
	$b_1\leq b_2\leq \ldots \leq b_{k}$, $b_0=-\infty$ and $b_{k+1}=\infty$. Here we assume
	without loss of generality that before the first breakpoint the function is
	negative because  we can always set the first breakpoint to be $-\infty$. It is
	clear that the function $\vec M$ is a continuous map and $\mathbb{\overline R}^{k+1}$ is a compact set,
	thus $\vec M\left(\mathbb{\overline R}^{k+1}\right)$ is a compact set.
	We also have that for every $\eps>0$ there is a point $\vec b\in \mathbb{\overline R}^{k+1}$
	such that $|\dotp{\vec M(\vec b)}{\vec e_i} |\leq \eps$. Thus, from compactness, we have that there
	exists a point $\vec b^*\in \mathbb{\overline R}^{k+1}$ such that $\vec M(\vec b^*)=\vec 0$.
	This completes the proof.
\end{proof}

 \section{SQ Lower Bound for ReLU Regression} \label{lb:relus}
In this section, we give the proof of Theorem~\ref{thm:ReLU}.
To prove our theorem, we construct a class $\mathcal{F}$ of Boolean-valued functions 
that is SQ hard to weakly learn. We use the SQ hardness of $\mathcal{F}$
to show that agnostically learning a ReLU with respect the square loss is also SQ hard. 

Our proof depends critically on the following technical result.

\begin{proposition}\label{prop:main-relu} 
For all $k \geq 1$, there exists an $O(k)$-piecewise constant function $f: \R \to \{\pm 1\}$ 
such that $\E_{z \sim \normal(0,1)}[f(z)z^t]=0$ for every non-negative integer $t\leq k$, 
and $\E_{z \sim \normal(0,1)}[f(z)\relu(z)] \geq 1/\poly(k)$.
\end{proposition}

\noindent The proof of Proposition~\ref{prop:main-relu} is given in Section~\ref{ssec:prop-main-relu}.

\medskip

We start with a brief overview of the proof.
To prove Proposition~\ref{prop:main-relu}, 
we first make essential use of Legendre polynomials to construct 
an explicit function $f':\R \to [-1,1]$ with the correct properties. By
rounding $f'$ to a Boolean-valued function, it is not hard to show that for
every $\eps>0$ there is a Boolean-valued function whose first $k$ moments are at
most $\eps$, and whose correlation with a ReLU is at least $1/\poly(k)$.
Using a slight variation of the techniques from Lemma~\ref{lem:main_diff}, 
we can obtain a function with these moments that is $O(k)$-piecewise constant.
Taking a limit of such functions with $\eps$ tending to $0$ gives the proposition.

The hard family of functions to learn will then be $\mathcal{F} = \{ F_{\bv} \}_{\bv \in S}$,
where $S$ is a set of nearly orthogonal unit vectors, and
$F_{\bv}(\bx) = C \cdot f(\langle \bv, \bx \rangle)$, 
for $C>0$ some appropriately chosen number of size polynomial in $k$. 
Since $f(\langle \bv, \bx \rangle)$ correlates with $\relu(\langle \bv, \bx \rangle)$, 
taking $C$ sufficiently large, we have
\[
\E\left[ \left( F_{\bv}(\bx)-\relu(\langle \bv, \bx \rangle) \right)^2 \right] < \E[F_{\bv}(\bx)^2] \left(1 - 1/\poly(k)\right) \;.
\]
If $\eps$ is a sufficiently small polynomial in $1/k$, any learner would need to return
a function $g$ where $\E[(F_{\bv}(\bx)-g(\bx))^2] < \E[F_{\bv}(\bx)^2] (1 - 1/\poly(k))$.
This implies both that $\E[(1/C^2) g^2]$ is not too large and that $g$ correlates
non-trivially with the Boolean-valued function $(1/C) F_{\bv}$. 
However, since the class of functions $\{ (1/C) F_{\bv} \}_{\bv \in S}$ has large SQ-dimension 
(because of the moment matching property established by Proposition~\ref{prop:main-relu}), 
Theorem~\ref{thm:ReLU} follows from Lemma~\ref{lem:fieldman_lemma}.

\medskip

We now give the proof of Theorem~\ref{thm:ReLU}, 
assuming Proposition~\ref{prop:main-relu} 

\begin{proof}[Proof of Theorem~\ref{thm:ReLU}]
The proof follows using the same construction as in Theorem~\ref{thm:LTF}, 
but using the $O(k)$-piecewise constant function $f$ from Proposition~\ref{prop:main-relu}.
Let $C(k)$ be a constant that depends on $k$ and $\mathcal{F}_k$ be the family of $O(k)$-decision lists of halfspaces, 
where each $F_{\vec v}\in \mathcal{F}_k$ has the form $ F_{\vec v}(\vec x)=C(k)\cdot f(\dotp{\vec v}{\vec x})$, 
for a unit vector $\vec v\in S$, where we use the set $S$ from Lemma~\ref{lem:near-orth-vec}. 
Let $\cal A$ be an agnostic SQ learner for ReLUs under Gaussian marginals.
We feed $\cal A$ a set of i.i.d. labeled examples from an arbitrary function $F_{\vec v}\in \mathcal{F}_k$. 
By definition, algorithm $\cal A$ computes a hypothesis $h:\R^d\mapsto \R$ such that 
\[\E_{\bx \sim \normal(\vec 0, \vec I)}[(h(\x)-F_{\vec v}(\x))^2] \leq  
\inf_{f \in \mathcal{C}_{\relu}}\E_{\bx \sim \normal(\vec 0,\vec I)}[(f(\x)- F_{\vec v}(\x))^2]+\eps\;.\]
We denote $\snorm{2}{g}^2=\E_{\bx \sim \normal(\vec 0, \vec I)}[g(\x)^2]$ for a function $g:\R^d\mapsto \R$. 
Let $C(k)=\frac{\snorm{2}{\relu}^2 }{\E_{\bx \sim \normal(\vec 0,\vec I)}[f(\dotp{\vec x}{\vec v})\relu(\dotp{\vec x}{\vec v})]}$. 
Then we have that
\begin{align*}
\E_{\bx \sim \normal(\vec 0, \vec I)}[(\relu(\dotp{\vec x}{\vec v})- F_{\vec v}(\x))^2]
&= \snorm{2}{F_{\vec v}}^2 + \snorm{2}{\relu}^2 - 2 \E_{\bx \sim \normal(\vec 0, \vec I)}[F_{\vec v}(\vec x)\relu(\dotp{\vec x}{\vec v})]\\
&=C^2(k)\snorm{2}{f}^2 -\snorm{2}{\relu}^2 \;.
\end{align*}
Furthermore, using that $\snorm{2}{f}^2=1$ and $\snorm{2}{\relu}^2=1/2$, 
if we choose $\eps=o(1/C^2(k))$, the algorithm returns a hypothesis such that
\[
\E_{\bx \sim \normal(\vec 0,\vec  I)}[(h(\x)- F_{\vec v}(\x))^2] \leq C^2(k)\left(1-\Omega(1/C^2(k))\right) \;.
\]
Thus, from the triangle inequality, we have that $\snorm{2}{h/C(k)}^2\leq 2\snorm{2}{f}^2$, and also 
\[2 \E_{\bx \sim \normal(\vec 0, \vec I)}\left[\frac{h(\x)}{C(k)} \frac{F_{\vec v}(\x)}{C(k)}\right]\geq \Omega(1/C^2(k))
+ \snorm{2}{h}^2/C^2(k) \geq \Omega(1/C^2(k))\;.\]
Finally,
\[
\E_{\bx \sim \normal(\vec 0, \vec I)}\left[\frac{h(\x)}{{\snorm{2}{h}}} \frac{F_{\vec v}(\x)}{{\snorm{2}{F_{\vec v}}}}\right]	
\geq  \frac 12 \E_{\bx \sim \normal(\vec 0,\vec I)}\left[\frac{h(\x)}{C(k)} \frac{F_{\vec v}(\x)}{C(k)}\right]\geq \Omega(1/C^2(k))\;.
\]
Let $h^*(\x)=\frac{h(\x)}{{\snorm{2}{h}}}$ and $F_{\vec v}^*(\x)=\frac{F_{\vec v}(\x)}{{\snorm{2}{F_{\vec v}}}} $. Then 
we have that $\E_{\bx \sim \normal(0, I)}\left[h^*(\x)F_{\vec v}^*(\x)\right]	\geq  \Omega(1/C^2(k))$.
Thus, using Proposition~\ref{prop:main-relu} to bound $C(k)$, we get that  
\[
\E_{\bx \sim \normal(\vec 0, \vec I)}\left[h^*(\x)F_{\vec v}^*(\x)\right]	\geq \Omega(1/\poly(k) )\;.
\]
Since the function $F_{\vec v}$ is an $O(k)$-decision list of halfspaces, 
we can apply Proposition~\ref{prop:LTF} to get that any SQ algorithm needs $d^{\Omega(k)}$ queries to 
$\mathrm{STAT}(d^{-\Omega(k)})$ to get 
$\E_{\bx \sim \normal(\vec 0, \vec I)}\left[h^*(\x)F_{\vec v}^*(\x)\right]\geq d^{-\Omega(k)}$. 
Thus, in order to learn with error $\opt+\eps$, for $\eps=o(1/\poly(k))$, 
the algorithm $\cal A$ needs to use $d^{\Omega((1/\eps)^c)}$ queries to 
$\mathrm{STAT}(d^{-\Omega((1/\eps)^c)})$, for a constant $c>0$.
\end{proof}

\subsection{Proof of Proposition~\ref{prop:main-relu}} \label{ssec:prop-main-relu}

To prove Proposition~\ref{prop:main-relu}, we first need to prove that there exists a function 
that has non-trivial correlation with the ReLU and whose first $k$ moments are zero.

We have the following crucial proposition.

\begin{proposition}\label{prop:relu} 
Let $k$ be a positive integer. There exists a function $f:\R\mapsto[-1,1]$ 
such that $\E_{z \sim \normal(0,1)}[f(z)z^t]=0$, for $0\leq t\leq k$, and 
$\E_{z \sim \normal(0,1)}[f(z)\relu(z)]>1/\poly(k)$.
\end{proposition}

The proof of Proposition~\ref{prop:relu} requires analytic properties of the Legendre poynomials
and is deferred to Section~\ref{ssec:proof-prop-relu}. 
In the main part of this subsection, we prove Proposition~\ref{prop:main-relu}, 
assuming Proposition~\ref{prop:relu}.

In the following lemma, we show that there exists a piecewise constant Boolean-valued function
with near-vanishing moments of degree at most $k$ and non-trivial correlation with the ReLU.

\begin{lemma}\label{lem:initial_relu} 
For any $\eps>0$ and any non-negative integer $k$, 
there exists a piecewise constant function $G:\R\mapsto\{\pm1 \}$ such that
$|\E_{z \sim \normal(0,1)}[G(z)z^t]|\leq \eps$, for $0\leq t\leq k$, and 
$\E_{z \sim \normal(0,1)}[G(z)\relu(z)]> 1/\poly(k) + O(\eps)$.
\end{lemma}

\begin{proof}
The proof is similar to the proof of Fact~\ref{fct:function}. 
The main difference here is that we need to construct a function that 
will also have non-trivial correlation with the ReLU. 
To do this, we use a probabilistic argument to show that there exists a function 
that is bounded in the range $[-1,1]$, that has non trivial correlation, 
and then we discretize the function as in Fact~\ref{fct:function}.
Let $f$ be the function from Proposition~\ref{prop:relu}. We split the interval
$[-1,1]$ into sub-intervals of length $\delta$ and we define the random
piecewise constant function $G$ in each interval $[z_0,z_0+\delta]$ as $G(z)= 1$
\nnew{with probability $(1 + \int_{z_0}^{z_0+\delta}f(z) \phi(z)\d z / \int_{z_0}^{z_0+\delta}\phi(z)\d z)/2$ and 
$G(z)=-1$ with probability $(1 - \int_{z_0}^{z_0+\delta}f(z) \phi(z)\d z / \int_{z_0}^{z_0+\delta}\phi(z)\d z)/2$. 
Thus, in each interval, we have $\E[G(z)]= \int_{z_0}^{z_0+\delta}f(z) \phi(z)\d z / \int_{z_0}^{z_0+\delta}\phi(z)\d z$. 
Then, for any $|z_0|\leq1-\delta$, we have that
\begin{align*}
\E\left[\int_{z_0}^{z_0 + \delta} G(z) z^t \phi(z)\d z \right]
&=\E\left[\int_{z_0}^{z_0 + \delta} G(z) (z_0+ O(\delta))^t \phi(z)\d z \right] = 
\int_{z_0}^{z_0 + \delta} f(z)\phi(z) (z_0+ O(\delta))^t\d z \\
& = \int_{z_0}^{z_0 + \delta}f(z)\phi(z)z^t \d z +  \int_{z_0}^{z_0 + \delta} t\cdot O((|z_0| +\delta)^{t-1}\delta) \d z \\
& = \int_{z_0}^{z_0 + \delta} f(z)z^t \phi(z)\d z  + t\cdot O\left((|z_0| +\delta)^{t-1}\delta^2\right)\;,
\end{align*}
where we used the Taylor series $z^t= (z_0+O(\delta))^t +  t\cdot O(\delta (|z_0| + \delta)^{t-1})$.
Thus, we obtain
\[ 
\E\left[\int_{-1}^{1} G(z) z^t \phi(z)\d z \right] =  \int_{-1}^{1} f(z)z^t \phi(z)\d z +t\cdot O(\delta)=t\cdot O(\delta)\;, 
\]
where we used that all the moments of $f$ with degree at most $k$ are zero and that $|z_0|+\delta\leq 1$.
Moreover, for $0\leq z_0\leq 1$, it holds that 
\begin{align*}
\E\left[\int_{z_0}^{z_0 + \delta} G(z)\relu(z) \phi(z)\d z \right]
&= \E\left[\int_{z_0}^{z_0 + \delta} G(z) z \phi(z)\d z \right]=  \int_{z_0}^{z_0 + \delta} f(z)\relu(z) \phi(z)\d z 
+ t\cdot O\left((|z_0| +\delta)^{t-1}\delta^2\right)\;,
\end{align*}
where we used the same method as before.}
Thus, it follows that
\[ 
\E\left[\int_{0}^{1} G(z) \relu(z) \phi(z)\d z \right] =  \int_{0}^{1} f(z) \relu(z) \phi(z)\d z +t\cdot O(\delta)>1/\poly(k) +t\cdot O(\delta)\;.
\]
Define the random variable $X_{i,t} = \int_{i\cdot \delta}^{i\cdot \delta + \delta} G(z) z^t \phi(z)\d z$ 
and $X_t=\sum_{i=-1/\delta}^{1/\delta} X_{i,t}$. Using Hoeffding bounds, we have that  
\[
\pr[|X_t-\E[X_{t}]|>\sqrt{\delta}\log(4/(t+1))] \leq 1/(2(t+1))\;,
\]
where we used that $|X_{i,t}|\leq \delta$. By the union bound, we get that
there is positive probability that all $X_t$ are within $\pm \sqrt{\delta}\log(4/(t+1))$ from the mean value, 
and thus, from the probabilistic method there is a function with this property.
Furthermore, we round the rest of the values of $G(z)$ as in the proof of
Fact~\ref{fct:function} (because $\relu(z)=z$ for $z>0$). Choosing the
correct constant value of $\delta$, the result follows.
\end{proof}
\begin{lemma}\label{lem:relu_decrease} 
Let $m$ and $k$ be positive integers such that $m>2k+5$ and $\eps>0$.
If there exists an $m$-piecewise constant function $f:\R \mapsto \{\pm 1\}$
such that $|\E_{z \sim \normal(0,1)}[f(z)z^t]|<\eps$ for all non-negative integers $t\leq k$, and 
$\E_{z \sim \normal(0,1)}[f(z)\relu(z)]> 1/\poly(k) + O(\eps)$, then there exists an at most
$(2k+5)$-piecewise constant function $g :\R \mapsto \{\pm 1\}$ such that $|\E_{z \sim \normal(0,1)}[g(z)z^t]|<\eps$ 
for all non-negative integers $t \leq k$ and $\E_{z \sim \normal(0,1)}[g(z)\relu(z)]> 1/\poly(k) + O(\eps)$.
\end{lemma}
\begin{proof}
	This proof is similar to the proof of Lemma~\ref{lem:main_diff}. 
	The only difference is that we have to keep also the correlation with the ReLU constant. 
	For completeness, we provide a full proof. 
	
	Let $\{b_1,b_2,\ldots,b_{m-1}\}$ be the breakpoints of $f$. 
	Let $F(z_1, z_2, \ldots, z_{m-1},z): \mathbb{\overline R}^m \mapsto \R$ 
	be an $m$-piecewise constant function with breakpoints on $z_1, \ldots, z_{m-1}$, 
	where $z_1<z_2< \ldots <z_{m-1}$ and $F(b_1, b_2, \ldots, b_{m-1},z)=f(z)$. 
	For simplicity, let $\vec z=(z_1, \ldots, z_{m-1})$ and  define 
	$M_i(\vec z)= \E_{z \sim \normal(0,1)}[F(\vec z, z)z^i]$, for all $0 \leq i \leq k$ and
	$M_c(\vec z)= \E_{z \sim \normal(0,1)}[F(\vec z, z)\relu(z)]$. Finally, let
	$\vec M(\vec z)=[M_0(\vec z), M_1(\vec z), \ldots, M_{k}(\vec z),  arM_c(\vec z)]^T$. 
	It is clear that 
	$$M_i(\vec z)=\sum_{n=0}^{m-1}\int_{z_n}^{z_{n+1}} F(\vec z, z) z^i \phi(z) \d z = 
	\sum_{n=0}^{m-1}a_n\int_{z_n}^{z_{n+1}} z^i \phi(z) \d z \;,$$ and 
	$$M_c(\vec z)=\sum_{n=0}^{m-1}\int_{z_n}^{z_{n+1}} F(\vec z, z) z \1\{z>0\} \phi(z) \d z
	=\sum_{n=0}^{m-1}a_n\int_{z_n}^{z_{n+1}} z \1\{z>0\}\phi(z) \d z \;,$$ 
	where $z_0= -\infty$, $z_m=\infty$, and $a_n$ is the sign of $F(\vec z,z)$ in the interval $(z_n,z_{n+1})$. 
	Note that $a_n=-a_{n+1}$ for every $0\leq n<m$.
	By taking the derivative of $M_c$ and $M_i$ in $z_j$, for $0<j<m$, we get that
	\[\frac{\partial}{\partial z_j} M_i(\vec z)= 2a_{j-1} z_j^i \phi(z_j) \quad
	\text{and}\quad 
	\frac{\partial}{\partial z_j} M_c(\vec z)= \begin{cases}
	2a_{j-1} z_j
	\phi(z_j),  \quad \text{if } a_j>0\\
	0,\quad \text{if}\quad  a_j\leq 0
	\end{cases} 
	\;.\]
	Combining the above, we get
	\[ \frac{\partial}{\partial z_j} \vec M(\vec z)= \begin{cases}
	2 a_{j-1}\phi(z_j) [1, z_j^1, \ldots , z_j^{k},z_j]^T,\quad \text{if }\quad
	z_j>0\\
	2  a_{j-1}\phi(z_j) [1, z_j^1, \ldots , z_j^{k},0]^T,\quad \text{if}\;\;\quad
	z_j\leq 0\;.
	\end{cases}
	\]
	We first work with the positive breakpoints. Let $i_0$ be the
	index of the first positive breakpoint and assume that the positive breakpoints
	are $m'>k+2$. We argue that there exists a vector $\vec u\in \R^{m-1}$
	such that  $\vec u=(0, \ldots, 0, \vec u_{i_0+1},\ldots,\vec u_{i_0+k+2},0,0,\ldots,0,1)$ 
	and the directional derivative of $\vec M$ in $\vec u$ is zero. 
	To prove this, we construct a system of linear equations, such that
	$\nabla_{\vec u} M_i(\vec z)=0$ for all $0\leq i\leq k$ and $\nabla_{\vec u} M_c(\vec z)=0$.
	Indeed, we have
$\sum_{j=1}^{k} \frac{\partial}{\partial z_j}  M_i(\vec z) \vec u_j
	= - \frac{\partial}{\partial z_{m-1}}  M_i(\vec z) $ or 
	$\sum_{j=1}^{k} a_{j-1} z_j^i \phi(z_j)\vec u_j=- a_{m-2} z_{m-1}^i \phi(z_{m-1})$ 
	and $\sum_{j=1}^{k} a_{j-1} z_j \phi(z_j)\vec u_j \1\{z_j\geq 0\} =- a_{m-2} z_{m-1} \phi(z_{m-1})\1\{z_{m-1}\geq 0\} $,
	which is linear in the variables $\vec u_j$. Note that the last equation is the same equation as the 
	$\nabla_{\vec u} M_1(\vec z)=0$, because we have positive breakpoints only.  
	Let $\hat{\vec u}$ be the vector with the variables from index $i_0+1$ to $i_0+k+2$,
	and let $\vec w$ be the vector of the right hand side of the system, i.e., $\vec w_i=- a_{m-2} z_{m-1}^i \phi(z_{m-1})$. 
	Then this system can be written in matrix form as $\vec V \vec D\hat{ \vec u}=\vec w$, 
	where $\vec V$ is the Vandermonde matrix, i.e., the matrix that is $\vec V_{i,j}=\alpha_i^{j-1}$, 
	for some values $\alpha_i$ and $\vec D$ is a diagonal matrix.
	In our case, $\vec V_{i,j}=z_i^{j-1}$ and $\vec D_{j,j}= 2 a_{j-1}\phi(z_j)$.
	It is known that the Vandermonde matrix has full rank iff for all $i\neq j$ we have $\alpha_i\neq \alpha_j$,
	which holds in our setting. Thus, the matrix $\vec V \vec D$ is nonsingular and there exists a solution to the equation.
Thus, there exists a vector $\vec u$ with our desired properties and, moreover,
	any vector in this direction is a solution to this system of linear equations.
	Note that the vector $\vec u$ depends on the value of $\vec z$,
	thus we consider $\vec u(\vec z)$ be the (continuous) function that returns a vector $\vec u$ given $\vec z$.

	We define a differential equation for the function $\vec v:\mathbb{\overline R}\mapsto\mathbb{\overline R}^{m-1}$, 
	as follows: $\vec v(0)= \vec b$, where $\vec b=(b_1, \ldots, b_{m-1})$, and
	$ \vec v'(T)=\vec u(\vec v(T))$ for all $T \in \mathbb{\overline R}$.
	If $\vec v$ is a solution to this differential equation, then we have:
	\[
	\frac{\d}{\d T} \vec M(\vec v(T))=\frac{\d}{\d \vec v(T)} \vec M(\vec v(T)) \frac{\d}{\d T} \vec v(T) 
	= \frac{\d}{\d \vec v(T)} \vec M(\vec v(T)) \vec u(\vec v(T)) =\vec 0 \;,
	\] 
	where we used the chain rule and
	that the directional derivative in the $\vec u(\vec v(T))$ direction is zero.
	This means that the function $\vec M(\vec v(t))$ is constant and, for all $0\leq j<k$, 
	we have $|M_j|< \eps$, because we have that $|\E_{z \sim \normal(0,1)}[F(z_1,\ldots, z_{m-1},z)z^t]|<\eps$. 
	Furthermore, since $\vec u(\vec v(T))$ is continuous in $\vec v(T)$, 
	this differential equation will be well founded and have a solution up until the point 
	where either two of the $z_i$ approach each other or one of the $z_i$ 
	approaches plus or to zero (the solution cannot oscillate, since $\vec v_{m-1}'(T)=1$ for all $T$).
	
	Running the differential equation until we reach such a limit, 
	we find a limiting value $\vec v^\ast$ of $\vec v(T)$ so that either:
	\begin{enumerate}[leftmargin=*]
		\item  There is an $i$ such that $\vec v_i^\ast=\vec v_{i+1}^\ast$, which
		gives us a function that is at most $(m-2)$-piecewise constant, i.e., taking $F(\vec v^\ast,z)$.
		\item  $\vec v_{m-1}^\ast = \infty$, which gives us an at most
		$(m-1)$-piecewise constant function, i.e., taking $F(\vec v^\ast,z)$.
		Since when the $\vec v_{m-1}^\ast= \infty$, the last breakpoint becomes $\infty$, we have one less breakpoint.
		\item $\vec v_{i_0+1}^\ast=0$, which gives us one less positive breakpoint.
	\end{enumerate}
	By iterating this method, we can get a function $f'$ that has at most $k+2$
	positive breakpoints. For the negative breakpoints, we work in a similar
	way, with the only difference that $\frac{\partial}{\partial z_j}M_c(\vec z)=0$,
	for all the negative breakpoints, and that the direction we increase has the
	form $\vec u=(-1,\vec u_1, \ldots, 0, \vec u_{k+2}, 0, \ldots, 0)$. 
	Thus, we get a function $g$ that has at most $2k+5$ breakpoints, 
	where we can get an extra breakpoint if $0$ is a breakpoint.
\end{proof}

\begin{proof}[Proof of Proposition~\ref{prop:main-relu}]
	For every $\eps>0$, using the function $f'$ from Lemma~\ref{lem:initial_relu} in
	Lemma~\ref{lem:relu_decrease}, we can obtain a function $f_{\eps}$ such that
	$|\E_{z \sim \normal(0,1)}[f_\eps(z)z^t]|\leq\eps$, 
	for every non-negative integer $t\leq k$ and 
	$\E_{z \sim \normal(0,1)}[f_{\eps}(z)\relu(z)]>1/\poly(k) + O(\eps)$.
	Moreover, the function $f_{\eps}$ is at most $(2k+5)$-piecewise constant. 
	
	Let $\vec M: \mathbb{\overline R}^{2k+5} \mapsto \R^{k+2}$, where 
	$M_i(\vec b)=\sum_{n=0}^{2k+5}(-1)^{n+1}\int_{b_n}^{b_{n+1}} z^i \phi(z) \d z$, for $0\leq i<k+2$, 
	and $M_{k+2}(\vec b)=\sum_{n=0}^{2k+5}(-1)^{n+1}\int_{b_n}^{b_{n+1}} \relu(z) \phi(z) \d z$, 
	where $b_0\leq b_1  \ldots \leq b_{2k+5}$, $b_0=-\infty$ and $b_{2k+6}=\infty$. 
	Here we assume without loss of generality that before the first breakpoint the function is
	negative, because we can always set the first breakpoint to be $-\infty$. It is
	clear that the function $\vec M$ is a continuous map and
	$\mathbb{\overline R}^{2k+5}$ is a compact set, thus 
	$\vec M\left(\mathbb{\overline R}^{2k+5}\right)$ is a compact set. 
	We also have that for every $\eps>0$, there is a point $\vec b\in \mathbb{\overline R}^{2k+5}$ 
	such that $|\dotp{\vec M(\vec b)}{\vec e_i} |\leq \eps$, for $0\leq i<k+2$, and 
	$\dotp{\vec M(\vec b)}{\vec e_{k+2}}>1/\poly(k) + O(\eps)$. 
	Thus, from compactness, we have that there exists a point 
	$\vec b^*\in \mathbb{\overline R}^{2k+5}$ such that 
	$|\dotp{\vec M(\vec b^*)}{\vec e_i} |=0$ for $0\leq i<k+2$, and 
	$\dotp{\vec M(\vec b^*)}{\vec e_{k+2}}>1/\poly(k)$.
\end{proof}

\subsection{Proof of Proposition~\ref{prop:relu}} \label{ssec:proof-prop-relu}
Below we
state some important properties of the Legendre polynomials that we use in our proofs.
\begin{fact}[see, e.g.,~\cite{Szego:39}]\label{fct:legendre} 
The Legendre polynomials, $P_n(z)$, for $n$ non-negative integer, satisfy the following properties:
\begin{itemize}
\item[(i)]$P_n(z)$ is a degree-$n$ univariate polynomial, with $P_0(z)=1$ and $P_1(z)=z$.\label{lege:1}
\item[(ii)] $\int_{-1}^{1} P_i(z)P_j(z) \d z =\delta_{ij} \frac{2}{2i+1}$, for all $i,j$ non-negative integers (orthogonality). \label{lege:2}
\item[(iii)] $|P_n(z)|\leq 1$, for all $|z|\leq 1$ (bounded).\label{lege:3}
\item[(iv)] $P'_{n}(z)=\sum_{t=0}^n \frac{2t+1}{2} P_t(z) $ (closed form of derivative).\label{lege:4}
\end{itemize}
\end{fact}

Using the Legendre polynomials, we can construct a function for which the first $k+1$
moments are zero and which has non-trivial correlation with the ReLU function.

 \begin{proof}[Proof of Proposition~\ref{prop:relu}]
	Define $f(z)=c\frac{\relu(z)-p(z)}{\phi(z)}\1\{z\in [-1,1]\}$, for a degree-$k$ polynomial
	$p(z)$ and a constant $c>0$. Then, we have
	\[ \E_{z \sim \normal(0,1)}[f(z)z^t]= c \int_{-1}^{1} (\relu(z)-p(z) )z^t\d
	z\;.\]
	We want $ \E_{z \sim \normal(0,1)}[f(z)z^t]=0$, thus we want to find a polynomial
	$p(z)$ such that
	\begin{equation}\label{eq:moment_bound}
	\int_{-1}^{1} \relu(z)z^t\d z=\int_{-1}^{1} p(z) z^t\d z\;. 
	\end{equation}Equation~\eqref{eq:moment_bound} is equivalent to saying that for all $0\leq t<k$, it holds
	\begin{equation}\label{eq:moment_bound2}
	\int_{-1}^{1} \relu(z)P_t(z)\d z=\int_{-1}^{1} p(z) P_t(z)\d z\;,
	\end{equation}
	because the Legendre polynomials of degree at most $k$
	span the space of polynomials of degree at most $k$. Using
	Fact~\ref{fct:legendre} (ii) and a standard computation involving
	orthogonal polynomials, gives that for 
	$p(z)=\sum_{t=0}^k \frac{2t+1}{2} P_t (z)  \int_{-1}^{1} \relu(z)P_t(z)\d z$, 
	Equation~\eqref{eq:moment_bound2} and Equation~\eqref{eq:moment_bound} hold. 
	We want the function $f$ to take values inside the interval $[-1,1]$. 
	To achieve this, we bound from above the constant $c$. It holds that 
	$\int_{-1}^{1} \relu(z)P_t(z)\d z\leq 2$, where we used Fact~\ref{fct:legendre} (iii) and 
	$|\relu(z)|\leq 1$ for $|z|\leq 1$. Moreover, we get that 
	\[
	|p(z)|\leq 2\sum_{t=0}^k \frac{2t+1}{2} |P_t (z)|\leq  k^2+2k\leq 2k^2 \;, 
	\]
	for all $|z|\leq 1$. Thus, it must hold that $c\leq g(1)/(2k^2+1)$, 
	and by taking $c=g(1)/(2k^2+1)$, we get that $|f(z)|\leq 1$.
	
	Next we prove that $\E_{z \sim \normal(0,1)}[f(z)\relu(z)]>1/\poly(k)$. We have
	that
	\[
	\E_{z \sim \normal(0,1)}[f(z)\relu(z)]=c \int_{-1}^1 \relu(z) (\relu(z)-p(z))\d
	z= c \int_{-1}^1  (\relu(z)-p(z))^2\d z\;,
	\]
	where we used that  $\int_{-1}^1 q(z) (\relu(z)-p(z))\d z=0$, for any polynomial $q$ of degree at most $k$, and thus it
	holds for $q(z)=p(z)$. Note that $|p'(z)|\leq 5 k^4$ and $|p''(z)|\leq 7k^6=:N$, because from
	Fact~\ref{fct:legendre} (iv), we have that  $|P'_n(z)|\leq 2 n^2$ and $|P''_n(z)|\leq 4 n^4$, for all $|z|\leq 1$.
	For $\eps>0$ sufficiently small, we then have
	\[
	\int_{-1}^1  (\relu(z)-p(z))^2\d z\geq \int_{-\eps}^\eps  (\relu(z)-p(z))^2\d z \;.
	\]
	Using the Taylor expansion of $p$, we get that there exists a linear function
	$L$, such that $p(z)=L(z)+ O(N \eps^2)$, for $|z|\leq \eps$. We thus have that
	\[ \int_{-\eps}^\eps  (\relu(z)-p(z))^2\d z=
	\int_{-\eps}^\eps  (\relu(z)-L(z)+O(N \eps^2) )^2\d z\;.\]
	Note that every function can be written as 
	$G(z)=G_{\mathrm{even}}(z) + G_{\mathrm{odd}}(z) $, where $G_{\mathrm{even}}(z)$ is the even part of $G$ and
	$G_{\mathrm{odd}}(z)$ is the odd part. For $\ell>0$, it holds that 
	\[ \int_{-\ell}^{\ell} G^2(z)\d z =  
	\int_{-\ell}^{\ell} \left( G_{\mathrm{even}}^2(z)+G_{\mathrm{odd}}^2(z) + 2G_{\mathrm{even}}(z)G_{\mathrm{odd}}(z) \right)\d z 
	\geq  \int_{-\ell}^{\ell} G_{\mathrm{even}}^2(z)\d z\;,\]
	where we used that $\int_{-\ell}^{\ell}
	G_{\mathrm{even}}(z)G_{\mathrm{odd}}(z)=0$. Using that $\relu(z)= |z|/2  + z/2$, it holds
	\[ \int_{-\eps}^\eps  (\relu(z)-L(z) +O(N \eps^2) )^2\d z 
	\geq  \int_{-\eps}^\eps  (|z|/2-L(0) +O(N \eps^2) )^2\d z\;, 
	\]
where we used that $L$ is linear, thus the even part is $L(0)$. \nnew{Choosing $\eps$ such that 
$N<\eps^{-1}/C$ for a large enough $C>0$, we have that 
$\left| \left|z\right|/2 - L(0)\right|\geq \eps/8$ for at least half of the interval $[-\eps,\eps]$. 
To prove this, note that we have two cases. First, if $L(0)>\eps/2$ or $L(0)\leq 0$, this holds trivially. 
Again in the other case trivially in half the points we have  $| |z|/2 - L(0)|\geq \eps/4$. 
Moreover, from the choice of $\eps$, we have that $N\eps^2\leq \eps/C$, 
thus $\left||z|/2-L(0)+O(N\eps^2)\right|\geq \left|\left|\left|z\right|/2-L(0)\right|- |O(N\eps^2)|\right|\geq \eps/8$ for at least half of the interval. Therefore, 
we have
\[
\int_{-\eps}^\eps  (|z|/2-L(0) +O(N \eps^2) )^2\d z\geq \Omega(\eps^3)\;.
\]
By our choice of $\eps$, we have
\[
c \int_{-1}^1  (\relu(z)-p(z))^2\d z\geq c\cdot \Omega(\eps^3)
\geq c\cdot \Omega(N^{-3})\geq \Omega(1/k^{20}) \;.
\]
This completes the proof.}
\end{proof}

\bibliography{allrefs}
\appendix
\bibliographystyle{alpha}

\appendix

\end{document}